\documentclass[11pt]{article}
\usepackage{geometry}
\usepackage[colorlinks=true, linkcolor=blue, citecolor=red, urlcolor=blue]{hyperref}
\usepackage{natbib}
\usepackage{graphicx}
\usepackage{amsthm}
\usepackage{amssymb}
\usepackage{amsmath}
\usepackage{enumerate}
\numberwithin{equation}{section}
\usepackage{longtable}
\newtheorem{definition}{Definition}

\newtheorem{proposition}{Proposition}

\title{DCS: A Unified Conditional Sensitivity Framework for\\ Cross-Modal Copyright Infringement Detection}
\author{
	\textbf{Xiafeng Man}\\
	Fudan University\\
	\texttt{xfmanacad@outlook.com}\\
}
	
\begin{document}
\maketitle

\begin{abstract}
	Currently, most foundation models can reproduce or strongly depend on copyrighted training content, but output similarity alone is insufficient for infringement detection, because similar outputs may also arise from public-domain concepts, common stylistic conventions, or ordinary statistical generalization. 
	
	In this paper, we develops a unified post-hoc detection framework that treats copyright infringement evidence as a counterfactual conditional distribution shift: a protected target is suspicious when the model's behavior under aligned conditions would change measurably if that target were included in, or removed from, the training process.
	
	We formalize this view through conditional differential privacy and introduce \textbf{D}ual-Branch \textbf{C}onditional \textbf{S}ensitivity (DCS), an operational statistic that measures the observable gap between two locally perturbed model states. 
	Specifically, the proposed DCS framework creates a learning branch and an unlearning branch around the deployed model, connects their displacement to the unavailable counterfactual retraining effect through influence-function analysis, and bounds the observable sensitivity by the counterfactual privacy-budget surrogate, local curvature, training-set scale, and perturbation step size. To distinguish target-specific memorization from generic fine-tuning instability, we further define a calibrated detection statistic that subtracts the sensitivity measured under orthogonal conditions.
	
	The DCS framework is instantiated for ridge-regularized linear regression, conditional diffusion models, autoregressive language models, and multimodal models. These instantiations show how the same principle can be evaluated through prediction gaps, image-embedding divergence, token-distribution or entropy shifts, and cross-modal representation changes. The resulting formulation provides an interpretable bridge between privacy leakage, memorization, and copyright detection across modalities, while making explicit the broader auditing assumptions under which post-hoc infringement detection is statistically meaningful.
	
\end{abstract}

\newpage
\tableofcontents

\newpage
\section*{Notation}
\addcontentsline{toc}{section}{Notation}

We use calligraphic letters such as $\mathcal{X}$ for spaces or abstract domains, bold symbols such as $\mathbf{p}$ for vectors or matrices, hats such as $\hat{\theta}$ for empirical estimators obtained from an optimization problem.

\begin{center}
	\begin{longtable}{p{0.25\textwidth}p{0.70\textwidth}}
		\hline
		\textbf{Symbol} & \textbf{Description} \\
		\hline
		\multicolumn{2}{l}{\textit{Math}}\\
		\hline
		$\mathbf{H}$ & Hessian matrix\\
		$\lambda$ & Eigenvalue \\
		$\lambda_{\max}(\mathbf{H})$ & Largest eigenvalue or local maximum curvature of the Hessian \\
		$\|\cdot\|_{\mathcal{F}}$ & Frobenius norm \\
		\hline
		\multicolumn{2}{l}{\textit{Models, Datasets, and Samples}} \\
		\hline
		$M$ & Randomized model or mechanism \\
		$\mathcal{X}$, $\mathcal{Y}$ & Model input and output space \\
		$S$ & Measurable output event \\
		$\mathbf{p} \in \mathcal{P}$ & Condition content and its space (e.g., prompts, queries or multimodal contents) \\
		$\mathcal{P}_c$, $\mathcal{P}_{\perp}$ & Target-aligned conditioning set and orthogonal conditioning set \\
		$\theta_D \in \Theta$ & Model parameter trained on dataset $D$ and its space \\
		$D$ (e.g., $D^+$, $D^-$) & Training dataset or locally perturbed model state shorthand \\
		$D'$ & Neighboring dataset differing from $D$ in one target sample or concept \\
		$z_i=(x_i,y_i)$ & A training data point with input $x_i$ and output $y_i$ \\
		$U(\cdot)$ & Semantic neighborhood \\
		$\Phi$ & Query function \\
		\hline
		\multicolumn{2}{l}{\textit{Losses, Influence, and Sensitivity}} \\
		\hline
		$\mathcal{L}(\cdot)$ & Loss function \\
		$R(\theta)$ & Empirical risk \\
		$\mathcal{I}_{\mathrm{up}}$ & Influence of upweighting a training point\\
		$\eta$ & Fine-tuning step size \\
		$I$ & Fine-tuning branch indicator \\
		$GS(M)$, $LS(M,D)$ & Global and local sensitivity of a query or mechanism \\
		$CS(z_c,\mathbf{p})$ & Conditional sensitivity with target concepts and conditions \\
		$S_c(z_c)$, $S_{\perp}(z_c)$ & Aligned sensitivity and orthogonal background sensitivity \\
		$CDS(z_c)$, $\widetilde{CDS}(z_c)$ & Calibrated detection statistic and normalized calibrated statistic \\
		$\Delta\theta$ & Parameter displacement \\
		\hline
		\multicolumn{2}{l}{\textit{Probability and Privacy}} \\
		\hline
		$\Pr[\cdot]$ & Probability of an event \\
		$q(\cdot)$ & Density or probability mass function \\
		$\epsilon$, $\delta$ & Differential privacy budget and its failure probability \\
		$\widehat{\epsilon}(z_c,\mathbf{p})$ & Observation-space surrogate of the conditional privacy budget \\
		$f(z_c)$ & Confidence score \\
		\hline
	\end{longtable}
\end{center}

\newpage
\section{Introduction}
\label{sec: introduction}

Due to the effect of scaling law, most large artificial intelligence models are increasingly trained on heterogeneous corpora containing text, images, code, audio, video, and multimodal associations collected at web scale. Their practical value comes from the ability to generalize from this data, but the same training process can also create model behaviors that reproduce, reconstruct, or strongly depend on protected expressive content. 

Prior studies~\cite{10.1145/3357713.3384290, carlini2021extracting, carlini2023extracting} on memorization and extraction have shown that rare or distinctive samples can become disproportionately influential in learned models, especially in long-tail regions of the data distribution. Diffusion models raise a particularly visible version of this concern, because generated images may replicate training examples or protected visual styles in ways that resemble digital forgery~\cite{somepalli2023diffusion, zhou2023copyscopemodellevelcopyrightinfringement, wang2023diagnosis}; related memorization risks are also emerging in video diffusion models~\cite{chen2025investigatingmemorizationvideodiffusion}.

Copyright auditing~\footnote{In this paper, \emph{detection} denotes the technical task of estimating the evidence of copyright infringement, whereas \emph{auditing} denotes the broader investigative setting in which such detection is deployed.} therefore requires more than a raw comparison between a generated output and a reference work. Two outputs may be similar because the target work was memorized, but it may also results from generic concept expression or similar region of the training distribution. 
Recent benchmarks and evaluation studies on copyright show that infringement detection and mitigation mostly depend on target concepts and prompts~\cite{ma2024datasetbenchmarkcopyrightinfringement, wang2024evaluatingmitigatingipinfringement, xu2025largevisionlanguagemodelsdetect}. However, even when ordinary prompts do not directly extract a close copy, a model may also contain information about a protected work.
The central difficulty is that we must demonstrate that the target work has a measurable effect on the model's output behavior, not merely visual or textual similarity. 

Based on the analysis above and our previous study~\cite{Man_Wei_Chen_2026}, we formulate copyright infringement detection as a \textbf{counterfactual distributional problem}. As developed in Section~\ref{sec: problem-formulation}, we will compare the behavior of a model trained with or without target concepts. This framing separates explicit infringement from ordinary generalization, and points out the auditing constraints faced by copyright holders, model providers, and third-party auditors. In particular, a meaningful auditing procedure must operate under a \textbf{copyright trilemma}: the target must be sufficiently distinctive, the model must expose observable behavior, and the model may also be protected by mechanisms that suppress direct extraction~\cite{bourtoule2021machine, alberti2025dataunlearningdiffusionmodels, gandikota2023erasingconceptsdiffusionmodels, kumari2023ablatingconceptstexttoimagediffusion}.
Therefore, we define a clear operational boundary and necessary assumptions for the detection work.

To give this view a mathematical foundation, Section~\ref{sec: dp-view} reinterprets copyright memorization through differential privacy. Classical differential privacy~\cite{10.1007/978-3-540-79228-4_1, 10.1561/0400000042} provides statistical guarantees against the information an adversary can infer through the output of a randomized algorithm.
In the region of copyright infringement detection, we adapt this idea to conditional generation by defining \textbf{conditional privacy, conditional publicity, and (non-) infringement criteria} under prompts, queries, or multimodal conditions. In this view, the inclusion of an infringed target should produce large conditional privacy violation, while a non-infringed or irrelevant target should produce approximate invariance.

The main technical contribution appears in Section~\ref{sec: unified-sensitivity}. 
Directly training the counterfactual model from scratch is usually infeasible, especially when the original training data and training pipeline are unavailable. \textbf{D}ual-Branch \textbf{C}onditional \textbf{S}ensitivity (\textbf{DCS}) therefore employs a perturbation strategy. Specifically, we construct two local branches from the original parameter state: a learning branch that takes a small optimization step toward the target, and an unlearning branch that takes the opposite step. The divergence between the two branches under aligned conditions defines the conditional sensitivity score. Through the coupling of  influence function~\cite{3305381.3305576, park2023trakattributingmodelbehavior}, the local dual-branch gap acts as an operational surrogate for the unavailable counterfactual drift induced by including or excluding the target.

However, fine-tuning any large model can introduce global instability unrelated to the target--the global parameter shift. Fine-tuning models on the target sample will inevitably influence the generation of unrelated samples.
To address this problem, we define a \textbf{calibrated detection statistic} in Section~\ref{sec: unified-5}. It compares sensitivity under aligned conditions $\mathcal{P}_c$ with the one under orthogonal conditions $\mathcal{P}_{\perp}$. The aligned component estimates the model's response in the semantic neighborhood of the protected work, while the orthogonal component estimates background drift caused by the perturbation procedure itself. Their difference yields a more robust infringement signal.

Finally, in Section~\ref{sec: modality-specific}, we instantiate the framework across several model classes:
  In ridge-regularized linear regression, the sensitivity and privacy-budget surrogate admit a closed-form relationship controlled by residuals and Hessian geometry.
 In conditional diffusion models, sensitivity is measured as the representation-space gap (e.g., CLIP image embeddings~\cite{radford2021learningtransferablevisualmodels}) between learning and unlearning samplers under aligned prompts, following the empirical motivation of our prior work~\cite{Man_Wei_Chen_2026}.
 In casual language models, the sensitivity signal appears as a divergence in next-token distributions or entropy along target continuations, consistent with the transformer-based large language model paradigm~\cite{NIPS2017_3f5ee243, zhao2026surveylargelanguagemodels}.
 In multimodal models, we extend the framework to the cross-modal representations and attention-based associations, allowing the framework to capture protected bindings between names, prompts, images, styles, and other modalities.

\textbf{In summary, our work makes the following contributions:}
	
\begin{itemize}
	\item To our knowledge, we are the first to establish a unified theoretical foundation for the copyright detection view of differential privacy;
	\item We propose a theoretically grounded theory for the post-hoc detection of copyright infringement. Specifically, we first formulate the evidence of copyright infringement as a conditional distribution shift; then we propose that a meaningful auditing procedure must operate under a copyright trilemma and that detection work has operational boundaries.
	\item From the perspective of mathematics and statistics, we introduce dual-branch conditional sensitivity (DCS) statistic, a unified conditional sensitivity measure for memorization and publicity violations, based on dual-branch perturbations: learning and
	unlearning fine-tuning branches. We also define an orthogonally calibrated detection score that greatly discounts generic fine-tuning drift;
	\item We instantiate our framework in different modalities, such as linear model, diffusion model, language model, and multimodal model.
	
\end{itemize}

\newpage
\section{Problem Formulation}
\label{sec: problem-formulation}

\subsection{Counterfactual Distribution Shift}

Defining the boundary between copyright infringement and legitimate model generalization remains a challenge in the discipline of artificial intelligence and law regulations, because output similarity alone is not a causal notion. 
When a model generates text or image output, it is inherently difficult to evaluate infringement based solely on raw, unfeatured outputs. 
It may arise from public-domain concepts, common stylistic conventions, or independent statistical convergence, rather than model memorization or direct replication.
Therefore, a static comparison between a generated output and a copyrighted reference is insufficient for determining whether the protected work had a concrete effect on the model.

Therefore, we formulate the detection of copyright infringement as \emph{a counterfactual distributional problem}. Let $x_c$ denote a copyrighted sample or concept, and let $M_D$ denote a model trained on dataset $D$. The relevant question is not merely whether some output of $M_D$ resembles $x_c$, but whether the conditional output distribution of the model would change if $x_c$ were removed from, or added to, the training process. In this view, infringement evidence corresponds to a measurable behavioral dependence of the model on the target copyrighted data.

The proposed formulation consists of two components:

\begin{itemize}
	\item \textbf{Counterfactual Analysis}: The model is evaluated across two states: one in which the target copyrighted data or concept is included in the training set, such as $D_{\mathrm{train}} \cup \{x_c\}$, and one in which it is strictly omitted, such as $D_{\mathrm{train}} \setminus \{x_c\}$.
	\item \textbf{Shift Quantification}: The behavioral difference between these states is measured through controlled conditioning spaces, such as prompts, queries, or multimodal contexts, by estimating the statistical distance between the output distributions.
\end{itemize}

With these formulations, a large counterfactual shift implies that the target data greatly influences model's behavior, therefore suggesting potential infringement. Conversely, approximate invariance suggests only generalization or no infringement.

\subsection{Copyright Trilemma}

\paragraph{Auditing Constraints}
Training datasets are often proprietary and not publicly disclosed due to legal, commercial, or privacy constraints. In some cases, even the model developers lack full visibility into the data sources and do not retain a complete record of all data sources, making direct access or full auditing impossible. 

\paragraph{Copyright Trilemma}
In such cases, copyright auditing is constrained by a three-way tension among copyright holders, auditors, and model providers.
Effective auditing requires the simultaneous satisfaction of these three parts, but in practice, three objectives cannot all be guaranteed, that is, any model can only satisfy at most two of the following objectives, as shown in Figure~\ref{fig: trilemma}:

\begin{itemize}	
	\item \textbf{High distinctiveness}: The audited target should contain protectable and non-generic expressive copyright characteristics, rather than merely reflecting public knowledge or highly homogeneous patterns.
	
	\item \textbf{Observable behavior}: The model's outputs should reveal enough information for auditors to estimate the internal distributional variations induced by the target data.
		
	\item \textbf{Robust generalization}: The model itself should generalize without exposing memorized training data. It should also protect copyright of training datasets by alignment, unlearning, concept erasure, or privacy-preserving mechanisms, therefore ensuring legitimate outputs.
\end{itemize}

\begin{figure}
	\centering
	\includegraphics[width=0.9\textwidth]{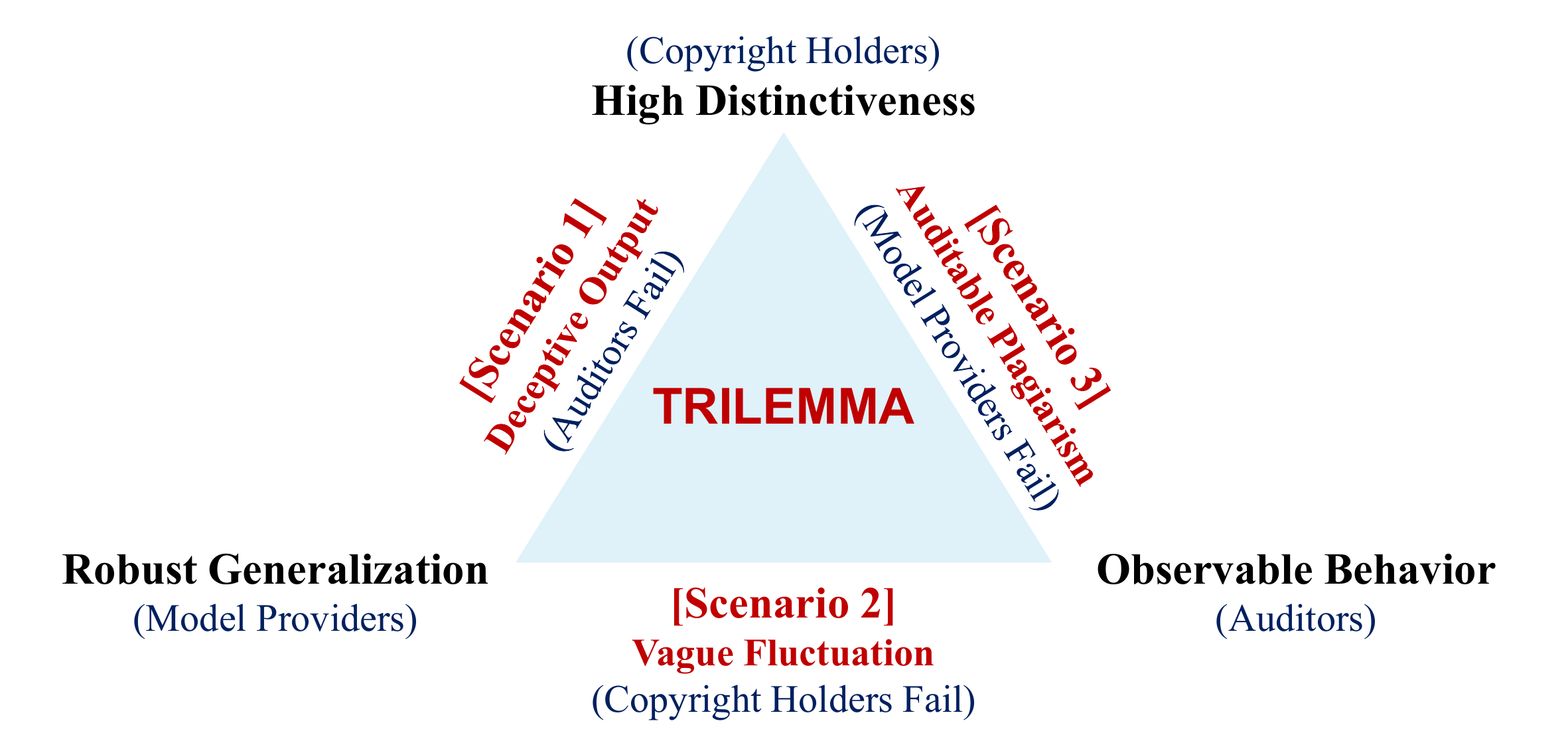}
	\caption{Trilemma of Copyright Auditing}
	\label{fig: trilemma}
\end{figure}

\paragraph{Copyright Trilemma Scenarios} We list three scenarios in copyright trilemma:

\begin{itemize}
	\item \textbf{Scenario 1: Deceptive Output.} In this scenario, a distinctive sample is included in training corpus, but the model employs robust post-hoc alignment, machine unlearning, differential privacy or other data protection mechanisms during training. These mechanisms may smooth observable output behaviors, thereby creating a hidden infringement state that standard output-based detection methods fail to capture.
	\item \textbf{Scenario 2: Vague Fluctuation.} Even for a secure and well-aligned model, auditors may possess sensitive tools capable of tracking subtle statistical fluctuations in the output space. 
	However, if the target consists of highly homogeneous or generic concepts, such as ordinary natural landscapes, it may lack sufficiently distinctive expression~\cite{ghorbani2019datashapleyequitablevaluation}. 
	In this case, any observed similarity can be plausibly attributed to legitimate generalization rather than copyright infringement.
	\item \textbf{Scenario 3: Auditable Plagiarism.} When a model is trained on a highly distinctive copyrighted work by copyright holders and operates without stringent privacy protection, its conditional output distribution may shift dramatically toward the long-tail memorized behavior~\cite{10.1145/3357713.3384290}. 
	This setting provides strong evidence for auditing, but it also indicates a failure of model privacy and security. 
	Unmasked memorization leaves extractable output behavior, making the model susceptible to membership inference and data extraction attacks~\cite{carlini2021extracting, carlini2023extracting, somepalli2023diffusion}. 
\end{itemize}

\subsection{Problem Assumptions}

The trilemma above implies that distributional shift detection is meaningful only under explicit operational assumptions. 
Our framework is not intended to declare infringement from arbitrary output resemblance. Instead, it tests whether a sufficiently distinctive target induces a statistically observable change in model behavior under controlled counterfactual perturbations.

First, the target data must display high \emph{semantic, stylistic, or structural salience}. Highly homogeneous or generic samples are difficult to isolate from the background training distribution because the marginal contribution of a single redundant data point may yield near invariance under counterfactual analysis.

Second, the model must exhibit \emph{observable behavior}. If specific algorithm fully masks the effect of the target data, then the relevant distributional change may be invisible to the parameter or output interrogation.

Third, the auditor must have access to \emph{a controlled conditioning space}. The space could be a set of prompts, captions, reference contexts, or multimodal queries semantically aligned with the target concept. These conditions define the neighborhood in which behavioral dependence is measured.

Under these assumptions, our goal is to construct a post-hoc detection statistic that estimates whether the inclusion or exclusion of a target copyrighted sample produces a measurable conditional distribution shift based on its infringement or non-infringement. 
The following sections will formalize it through differential privacy and a unified sensitivity measure.

\newpage
\section{Differential Privacy View of Copyright}
\label{sec: dp-view}

In this section, we formalize copyright memorization as a conditional privacy problem. We first introduce the notation and sample effects from the lens of influence function, and then reinterpret copyright exposure through differential privacy, conditional publicity, and its behavioral criteria.

\subsection{Notation and Problem Setup}
Let $\mathcal{X}$ denote the input space and $\mathcal{Y}$ denote the discrete or continuous output space. A foundation model parameterized by $\theta \in \Theta \subseteq \mathbb{R}^d$ is trained on a base dataset $D = \{z_1, z_2, \dots, z_n\}$, where each data point $z_i = (x_i, y_i) \in \mathcal{X} \times \mathcal{Y}$. 

For a point $z$ and parameter state $\theta$, let $\mathcal{L}(z; \theta)$ be the loss function. We assume that the empirical risk $R(\theta)$ is twice-differentiable and strictly convex in $\theta$, the objective function and its unique optimal estimator $\hat{\theta}$ are formulated as:
\begin{equation}
	R(\theta) = \frac{1}{n} \sum_{i=1}^n \mathcal{L}(z_i; \theta);
\end{equation}
\begin{equation}
	\hat{\theta} \triangleq \arg\min_{\theta \in \Theta} R(\theta) = \arg\min_{\theta \in \Theta} \frac{1}{n} \sum_{i=1}^n \mathcal{L}(z_i; \theta).
\end{equation}

In conditional generation scenarios, we denote $\mathbf{p} \in \mathcal{P}$ as a generalized conditioning prompt operator initialized from or semantically aligned with the neighborhood $U(z_c)$ of a target concept $z_c$ (e.g., stylistically, structurally, or contextually similar samples).

\subsection{Influence Functions}
Influence functions~\cite{3305381.3305576} estimate the impact of training points on model parameters without retraining the model from scratch. Scalable attribution methods such as TRAK~\cite{park2023trakattributingmodelbehavior} further motivate the use of local linearized approximations when direct counterfactual retraining is unavailable. For diffusion models, training-data attribution methods similarly aim to quantify how individual samples affect generated behavior~\cite{dai2023trainingdataattributiondiffusion, lin2025diffusionattributionscoreevaluating}. The change in model parameters due to removing a point $z$ from the training set is given by:
\begin{equation}
	\Delta \hat{\theta}_z = \hat{\theta}_{-z} - \hat{\theta} = \arg\min_{\theta \in \Theta} \frac{1}{n} \sum_{z_i \neq z} \mathcal{L}(z_i; \theta) - \arg\min_{\theta \in \Theta} \frac{1}{n} \sum_{i=1}^n \mathcal{L}(z_i; \theta).
\end{equation}
To avoid the computational cost of complete retraining, we compute the parameter change of $z$ by introducing a small perturbation $\epsilon$:
\begin{equation}
	\hat{\theta}_{\epsilon,z}=\arg\min_{\theta\in\Theta}\left[R(\theta)+\epsilon \mathcal{L}(z;\theta)\right].
\end{equation}

By forming a first-order linear approximation around $\hat{\theta}$~\cite{cook1982residuals}, the parameter change under upweighting $z$ is formalized as:
\begin{equation}
	\mathcal{I}_{\mathrm{up,params}}(z)\triangleq
	\left.\frac{d\hat{\theta}_{\epsilon,z}}{d\epsilon}\right|_{\epsilon=0}
	=-\mathbf{H}_{\hat{\theta}}^{-1}\:\nabla_{\theta}\mathcal{L}(z;\hat{\theta})
\end{equation}
where $\mathbf{H}_{\hat{\theta}}$ is the positive-definite Hessian matrix:
\begin{equation}
	\mathbf{H}_{\hat{\theta}}\triangleq\frac1n\sum_{i=1}^n\nabla_\theta^2\mathcal{L}(z_i;\hat{\theta}).
\end{equation}

Since removing a point $z$ is mathematically equivalent to shifting its sample weight by $\Delta \epsilon = -\frac{1}{n}$, the parameter shift can be linearly approximated as:
\begin{equation} 
	\label{eq: if_theta_approx}
	\Delta \hat{\theta}(z) \approx -\frac{1}{n} \mathcal{I}_{\mathrm{up,params}}(z) = \frac{1}{n} \mathbf{H}_{\hat{\theta}}^{-1} \nabla_{\theta} \mathcal{L}(z; \hat{\theta}).
\end{equation}
Correspondingly, the influence of sample $z$ on the model's loss prediction at point $z_{\mathrm{test}}$ can be estimated as:
\begin{equation}
	\mathcal{I}_{\mathrm{up,loss}}(z, z_{\mathrm{test}}) = - \nabla_{\theta} \mathcal{L}(z_{\mathrm{test}}; \hat{\theta})^\top \mathbf{H}_{\hat{\theta}}^{-1} \nabla_{\theta} \mathcal{L}(z; \hat{\theta}).
\end{equation} 

Despite theoretical elegance, influence functions suffer from the computational intractability of calculating $\mathbf{H}^{-1}$. Although prior work~\cite{basu2020secondordergroupinfluencefunctions, zheng2024intriguingpropertiesdataattribution, mlodozeniec2025influencefunctionsscalabledata} has developed approximations to the inverse Hessian matrix, these methods still require substantial memory and computation.

\subsection{From Privacy Bounds to Publicity Violations}
To reinterpret copyright memorization or infringement under a rigorous framework, we formalize model behaviors with differential privacy metrics.

\subsubsection{Differential Privacy}
Differential privacy is a formal notion of algorithmic privacy, which aims to prevent the release of private information~\cite{10.1007/978-3-540-79228-4_1, 10.1561/0400000042}.

Let $D_0 \in \mathcal{X}^m$ denote a database from the input domain $\mathcal{X}$, and two databases $D, D' \in \mathcal{X}^m$ are said to be neighbouring datasets if $d(D, D') \leq 1$, where $d$ represents the distance between two datasets. 
A randomized mechanism $M$ maps the input space $\mathcal{X}$ to a measurable output space $\mathcal{Y}$. 

Differential privacy is formally defined as follows:

\begin{definition}[Differential Privacy]
	A randomized mechanism $M: \mathcal{X}^m \rightarrow \mathcal{Y}$ is said to satisfy \emph{$(\epsilon, \delta)$-differential privacy} if for every pair of neighbouring databases $D$ and $D'$, and for every measurable output event $S \subseteq \mathcal{Y}$:
	\begin{equation} \label{eq: DP}
		\Pr[M(D) \in S] \leq e^{\epsilon} \cdot \Pr[M(D') \in S] + \delta
	\end{equation}
	with probability $1-\delta$, where $\epsilon$ is the privacy budget and $\delta$ is a failure probability for the definition (typically $\delta \leq \frac{1}{n^2}$), which means the privacy guarantee cannot hold with probability $\delta$.
\end{definition}

For a query $M: D_0 \rightarrow \mathcal{R}$, the global sensitivity (GS) is the maximum distance for any two neighbouring datasets:
\begin{equation} \label{eq: DPGS}
	GS(M) = \max_{D, D': d(D,D') \leq 1} \left\|M(D) - M(D')\right\|_{\mathcal{F}}
\end{equation}
which is independent of the actual dataset being queried.

Another measure of sensitivity, local sensitivity of $M$ at $D: D_0$, fixes one of the two datasets $D$ to be the actual dataset being queried, and considers all of its neighbors:
\begin{equation} \label{eq: DPLS}
	LS(M, D) = \max_{D': d(D,D') \leq 1} \left\| M(D) - M(D') \right\|_{\mathcal{F}}.
\end{equation}

\subsubsection{Conditional Privacy and Publicity}
As most large models produce outputs under conditioning spaces, we first define conditional differential privacy:

\begin{definition}[Conditional Differential Privacy] 
	A randomized mechanism $M: \mathcal{X}^m \times \mathcal{P} \rightarrow \mathcal{Y}$ conditioned on $\mathbf{p} \in \mathcal{P}$ is said to satisfy \emph{$(\epsilon, \delta, \mathbf{p})$-conditional differential privacy} if for every pair of neighbouring databases $D$ and $D'$ that differ by at most one single element associated with the semantic neighborhood $U(\mathbf{p})$, and for every measurable output event $S \subseteq \mathcal{Y}$:
	\begin{equation} \label{def: cdp}
		\Pr[M(D,\mathbf{p}) \in S] \leq e^{\epsilon} \cdot \Pr[M(D',\mathbf{p}) \in S] + \delta
	\end{equation}
 	with probability $1-\delta$, where $\epsilon$ is the conditional privacy budget and $\delta$ is a failure probability for the definition.
\end{definition}

\begin{definition}[Conditional Publicity]
	A randomized mechanism $M: \mathcal{X}^m \times \mathcal{P} \rightarrow \mathcal{Y}$ conditioned on $\mathbf{p} \in \mathcal{P}$ is said to satisfy \emph{$(\epsilon, \mathbf{p})$-conditional publicity} if there exist neighbouring training datasets $D$ and $D'$ that differ in at most one single element, and there exists at least one measurable output event $S \subseteq \mathcal{Y}$ such that:
	\begin{equation}
		\Pr[M(D, \mathbf{p}) \in S] 
		> e^\epsilon \cdot \Pr[M(D', \mathbf{p}) \in S] 
		\gg \Pr[M(D', \mathbf{p}) \in S]
	\end{equation}		
	where $\epsilon > 200$ indicates a substantial violation of privacy, i.e., the model output strongly depends on the presence of a specific concept in the training set.
\end{definition}
The notion of conditional publicity implies that an algorithm could significantly alter its relevant outputs when a single concept in the training dataset is modified. On the contrary, an input that the model has not seen, namely, non-infringed data, could be relatively private:

\begin{proposition}
	Let $x$ be a data point that does not appear in either $D$ or its neighbouring dataset $D'$. In the idealized case where the conditional output distribution of $M$ is independent of $x$ and its semantic neighborhood, the mechanism $M$ satisfies $(0,0,\mathbf{p})$-conditional differential privacy with respect to $x$ for any condition $\mathbf{p}$.
\end{proposition}

The requirement $x \notin D$ and $x \notin D'$ is often too strict in realistic models, where semantic neighbors of $x$ may still influence model behavior even if $x$ itself is not present in the training set.
To address this, we introduce a relaxed notion of approximate relative privacy:

\begin{definition}[Approximate Relative Privacy]
	Let $x$ be a data point such that $x \notin D$ and $x \notin D'$, and let $S \subseteq \mathcal{Y}$ be a measurable output event under conditions in $U(\mathbf{p})$.
	We say the model $M$ satisfies $(\epsilon, \delta, \mathbf{p})$-approximate relative conditional differential privacy with respect to $x$ if:
	\begin{equation}
		\Pr[M(D, \mathbf{p}) \in S] \leq e^{\epsilon} \cdot \Pr[M(D', \mathbf{p}) \in S] + \delta
	\end{equation}
	for small $\epsilon > 0$ and $\delta \approx 0$, indicating negligible influence from $x$ or its semantic neighborhood.
\end{definition}

The above proposition and definition formalize the notion that non-infringed data should not exert measurable influence on the model's output distribution. In contrast to infringed data that violates conditional privacy guarantees, non-infringed data satisfies strict or approximate relative privacy. 

This theoretical distinction lays the foundation for our detection criterion: if a concept induces statistically significant behavioral change in model outputs, it suggests potential memorization and copyright infringement.

\subsection{Formal Criteria for Infringement and Non-Infringement}

Copyright infringement may occur when an algorithm $M$ is trained on a dataset $D$ that includes a subset of copyrighted or unauthorized data samples
$D_C = \{ x_{c_1}, x_{c_2}, \dots, x_{c_n} \} \subset D$, without permission.
Evidence for such an infringement claim is typically established by demonstrating that the model's outputs reproduce or are substantially similar to the protectable expressive elements of the samples in $D_C$.

Building upon the formalism of conditional differential privacy and publicity, we now define mathematical behavioral criteria for: when an algorithm $M$ exhibits evidence of infringement upon a specific data point, and when it does not, based on observable model behavior.

\begin{definition}[Copyright Infringement]
	Let $x_c \in D_C$ denote a copyrighted data point, and let $\mathbf{p}$ be a conditional input semantically aligned with $x_c$. We say that model $M$ trained on $D$ exhibits infringement evidence upon $x_c$ if there exists a measurable output event $S \subseteq \mathcal{Y}$ such that:
	\begin{equation}
		\Pr[M(D, \mathbf{p}) \in S] \gg \Pr[M(D', \mathbf{p}) \in S]
	\end{equation}
	where $D' = D \setminus \{x_c\}$ is a neighboring dataset.
\end{definition}

This definition implies that the relevant output of $M$ is significantly influenced by the presence of $x_c$, thereby providing evidence of copyright exposure induced by memorization.

\begin{definition}[Copyright Non-Infringement]
	Let $x$ be a non-infringed data point such that $x \notin D$ for all training datasets considered. We say that model $M$ does not exhibit infringement evidence upon $x$ if for any input condition $\mathbf{p}$ and for every measurable output event $S \subseteq \mathcal{Y}$ such that:
	\begin{equation}
		\Pr[M(D, \mathbf{p}) \in S] = \Pr[M(D', \mathbf{p}) \in S] .
	\end{equation}
	
	To allow for a relaxed setting, we say that $M$ satisfies approximate non-infringement evidence if the following $(\epsilon, \delta, \mathbf{p})$-differential privacy condition holds:
	\begin{equation}
		\Pr[M(D, \mathbf{p}) \in S] \leq e^\epsilon \cdot \Pr[M(D', \mathbf{p}) \in S] + \delta .
	\end{equation}
\end{definition}

This definition states that the output distribution is invariant, or approximately invariant, with respect to the inclusion or exclusion of $x$, thereby indicating that $x$ has no measurable influence on model behavior under specific conditions.

\newpage
\section{Dual-Branch Conditional Sensitivity (DCS)}
\label{sec: unified-sensitivity}

To establish a rigorous mathematical foundation for post-hoc copyright infringement detection across modalities, we introduce the \textbf{Dual-Branch Conditional Sensitivity (DCS)} framework. 
Our aim is to transform the privacy budget from previous sections into an operational statistic that can be estimated without retraining the foundation model from scratch. 

DCS is used for denoting complete detection framework, rather than a single scalar quantity. It consists of three levels: (i) \emph{dual-branch perturbation}, which constructs learning and unlearning branches around the target; (ii) \emph{pointwise conditional sensitivity} $CS(z_c,\mathbf{p})$, which measures the branch divergence under one aligned or orthogonal condition; and (iii) the calibrated DCS detection statistic, $CDS(z_c)$, which aggregates pointwise sensitivity over aligned conditions and subtracts or normalizes by orthogonal background sensitivity.

The framework proceeds by first defining a counterfactual privacy budget (Section~\ref{sec: unified-1}), then constructing dual-branch perturbations (Section~\ref{sec: unified-2}), relating their displacement to counterfactual influence through influence functions (Section~\ref{sec: unified-3}), projecting the resulting parameter variation into an observable output space (Section~\ref{sec: unified-4}), and finally calibrating the resulting detection statistic against orthogonal conditions (Section~\ref{sec: unified-5}).

\subsection{Counterfactual Privacy Budget} \label{sec: unified-1}

In the canonical framework of Conditional Differential Privacy (CDP), a randomized generative mechanism $M$ can be formalized as $ M:\mathcal{D}\times\mathcal{P}\to \mathcal{Y} $, where $\mathcal{D}$ denotes the space of training datasets, $\mathcal{P}$ denotes the conditioning space, and $\mathcal{Y}$ denotes the output space.
For a dataset $D$, we write $\theta_D\in\Theta$ for the corresponding trained parameter state.

Let $D$ and $D'$ be neighboring datasets that differ only in one target copyrighted concept $z_c$, such that $D \setminus D' = \{z_c\}$.
For a given condition $\mathbf{p} \in \mathcal{P}$ and a measurable output event $S \subseteq \mathcal{Y}$, the conditional empirical privacy budget captures the maximum log-likelihood ratio triggered by the counterfactual inclusion or exclusion of $z_c$. Derived from Eq.~\ref{eq: DP}, we define the \emph{conditioned counterfactual privacy budget} as:
\begin{equation}
	\label{def: budget}
	\epsilon(z_c, \mathbf{p}) \triangleq \sup_{S \subseteq \mathcal{Y}} \left| \ln \frac{\Pr[M(D, \mathbf{p}) \in S]}{\Pr[M(D', \mathbf{p}) \in S]} \right|.
\end{equation}
When the conditional outputs admit densities or probability mass functions $q_D(y|\mathbf{p})$ and $q_{D'}(y|\mathbf{p})$ with respect to a common base measure, Eq.~\ref{def: budget} can be written as:
\begin{equation}
	\epsilon(z_c,\mathbf{p})
	=
	\operatorname*{ess\,sup}_{y\in\mathcal{Y}}
	\left|
	\ln \frac{q_D(y|\mathbf{p})}{q_{D'}(y|\mathbf{p})}
	\right|.
\end{equation}
For practical estimation in a chosen observation, feature, score, or log-density space, we use the norm-based surrogate:
\begin{equation}
	\label{def: budget_widehat}
	\widehat{\epsilon}(z_c,\mathbf{p})
	\triangleq
	\left\|
	\ln \frac{q_D(\cdot|\mathbf{p})}{q_{D'}(\cdot|\mathbf{p})}
	\right\|_{\mathcal{F}},
\end{equation}
where $\|\cdot\|_{\mathcal{F}}$ denotes the norm induced by the selected observable representation.
Unlike the essential-supremum privacy budget $\epsilon(z_c,\mathbf{p})$, the surrogate $\widehat{\epsilon}(z_c,\mathbf{p})$ measures the aggregate log-density drift in the chosen observation space and is used as the operational quantity in the sensitivity bound below.

\subsection{Dual-Branch Perturbation} \label{sec: unified-2}
In practical digital forensics, direct evaluation of Eq.~\ref{def: budget} is computationally infeasible because the counterfactual trained parameter state $\theta_{D'}$ cannot be directly obtained without retraining model from scratch. To avoid the computation process, our framework operationalizes the local effect of $z_c$ by applying two opposite optimization updates to the pretrained weight state $\theta_D$. Such lightweight local adaptation can be implemented with ordinary fine-tuning or parameter-efficient updates such as LoRA~\cite{hu2022lora}.

Given a target sample $z_c$, we initiate two parallel, lightweight fine-tuning branches to construct explicit parameter variants. The learning branch, where the indicator $I = +1$, takes the step to learn the target concept:
\begin{equation} \label{eq: learning_branch}
	\theta_{D^+} = \theta_D - \eta \nabla_\theta \mathcal{L}(z_c; \theta_D).
\end{equation}
Concurrently, the unlearning branch, where the indicator $I = -1$, takes the opposite step to unlearn that concept:
\begin{equation} \label{eq: unlearning_branch}
	\theta_{D^-} = \theta_D + \eta \nabla_\theta \mathcal{L}(z_c; \theta_D),
\end{equation}
where $\mathcal{L}(z_c; \theta_D)$ denotes the task-specific loss such as causal language modeling or score-matching, and $\eta > 0$ represents a small local step size.
Here, we use $D^+$ and $D^-$ as shorthand for the two locally perturbed model states, not as fully retrained datasets. But its ideal state is the same as the fully retrained one.

Since $\theta_{D^+}$ and $\theta_{D^-}$ are explicit parameter states produced by these two local updates, their displacement in the parameter space $\Theta$ provides an operational quantity that can be compared with the unavailable counterfactual drift induced by removing $z_c$:
\begin{equation}
	\label{eq: absolute_displacement}
	\Delta \theta \triangleq \theta_{D^+} - \theta_{D^-} = -2\eta \nabla_\theta \mathcal{L}(z_c; \theta_D)
\end{equation}

\subsection{Influence-Function Coupling} \label{sec: unified-3}

To relate the operational displacement to the counterfactual removal of $z_c$, we use the influence-function approximation from Eq.~\ref{eq: if_theta_approx}. If $D'=D\setminus\{z_c\}$ denotes the neighboring dataset with the target removed, then:
\begin{equation}
	\label{eq: influence_bridge}
	\theta_{D'} - \theta_D \approx \frac{1}{n} \mathbf{H}^{-1} \nabla_\theta \mathcal{L}(z_c; \theta_D),
\end{equation}
where $n$ is the cardinality of the base training dataset, and $\mathbf{H} = \nabla^2 R(\theta_D)$ is the Hessian matrix:
\begin{equation}
	\label{eq: isolated_gradient}
	\nabla_\theta \mathcal{L}(z_c; \theta_D)
	\approx n\mathbf{H}(\theta_{D'}-\theta_D)
	=-n\mathbf{H}(\theta_D-\theta_{D'}).
\end{equation}
Substituting Eq.~\ref{eq: isolated_gradient} into Eq.~\ref{eq: absolute_displacement} yields:
\begin{equation}
	\label{eq: final_param_delta}
	\Delta \theta \approx 2n\eta \cdot \mathbf{H} (\theta_D - \theta_{D'}).
\end{equation}

In summary, Eq.~\ref{eq: final_param_delta} shows that the dual-branch displacement provides an operational counterpart to the counterfactual parameter change caused by removing the target sample $z_c$. Under the approximation of influence function, $\Delta\theta$ is proportional to the Hessian-weighted difference between $\theta_D$ and $\theta_{D'}$. Therefore, the displacement is not an arbitrary perturbation, but a curvature-aware manifestation of the target sample's contribution to the fine-tuned model.

This interpretation highlights that the influence of $z_c$ depends jointly on its gradient contribution and on the local curvature of the empirical risk landscape. Samples whose gradients align with sensitive directions of the model can induce larger operational displacements, while gradients along less sensitive or highly constrained directions may produce smaller changes. Hence, the displacement serves as a practical proxy for measuring the counterfactual influence of $z_c$ on the model parameters.

\subsection{Conditional Sensitivity} \label{sec: unified-4}

\paragraph{Definition}
To translate internal parametric variations into measurable evidence, we project the dual-branch parameters back into the output space $\mathcal{Y}$ under the generalized condition $\mathbf{p}$. We define the \emph{Conditional Sensitivity} $CS(z_c, \mathbf{p})$ via the Frobenius norm of the macroscopic output behavior divergence:
\begin{equation}
	\label{def: conditional_sensitivity}
	CS(z_c,\mathbf{p}) \triangleq d_{\mathcal{O}}\left[ O(D^+,\mathbf{p}), O(D^-,\mathbf{p}) \right],
\end{equation}
where $O \in \mathcal{O}$ is the observable representation, such as Frobenius, KL divergence and cosine, $d_{\mathcal{O}}$ is the distance of certain modality. In the following passage, we will mostly use the Frobenius as the instantiation:
\begin{equation}
	CS(z_c, \mathbf{p}) = \left\| M(D^+, \mathbf{p}) - M(D^-, \mathbf{p}) \right\|_{\mathcal{F}}
\end{equation}

\paragraph{First-Order Bound} The conditional sensitivity admits a first-order upper bound that connects the observable dual-branch output divergence to the counterfactual exposure of the target sample and the local curvature of the empirical risk landscape:

\begin{proposition}[Conditional Sensitivity Bound]
	Assume that $M(D,\mathbf{p})$ is locally differentiable in $\theta$ around $\theta_D$, that the training step size $\eta$ is sufficiently small for a first-order approximation, and that the output representation is chosen in a feature, score, logit, or log-density space where the norm $\|\cdot\|_{\mathcal{F}}$ is well-defined. Assume further that the norm-based log-density drift in this representation is measured by $\widehat{\epsilon}(z_c,\mathbf{p})$ as defined in Eq.~\ref{def: budget_widehat}. Then the conditional sensitivity satisfies the first-order bound:
	\begin{equation}
		\label{eq: conditional_sensitivity_bound}
		CS(z_c,\mathbf{p})
		\lesssim
		2n\eta \cdot
		\lambda_{\max}(\mathbf{H}) \cdot
		\widehat{\epsilon}(z_c,\mathbf{p}).
	\end{equation}
\end{proposition}

\begin{proof}
	Since the optimization step size $\eta$ is small, both perturbed parameter states remain in a local neighborhood of $\theta_D$ where the mechanism is well approximated by its first-order expansion:
	\begin{equation}
		M(D^+, \mathbf{p}) - M(D^-, \mathbf{p}) \approx \left[ \nabla_\theta M(D, \mathbf{p}) \right]^T \cdot (\theta_{D^+}-\theta_{D^-})
		= \left[ \nabla_\theta M(D, \mathbf{p}) \right]^T \cdot \Delta\theta.
	\end{equation}
	Substituting Eq.~\ref{eq: final_param_delta} gives:
	\begin{equation}
		CS(z_c,\mathbf{p}) \approx 2n\eta \cdot \left\| \left[ \nabla_\theta M(D,\mathbf{p}) \right]^T \cdot \mathbf{H}(\theta_D-\theta_{D'}) \right\|_{\mathcal{F}}.
	\end{equation}
	
	By applying the bounded operator norm inequality, the empirical Hessian matrix $\mathbf{H}$ can be structurally decoupled from the inner Frobenius matrix product and projected into the spectral domain as the local maximum eigenvalue, or local principal curvature scalar, $\lambda_{\max}(\mathbf{H})$:
	\begin{equation}
		\label{eq: decoupled_sensitivity}
		CS(z_c,\mathbf{p}) \le 2n\eta \cdot \lambda_{\max}(\mathbf{H}) \cdot \left\| \left[ \nabla_\theta M(D,\mathbf{p}) \right]^T \cdot (\theta_D-\theta_{D'}) \right\|_{\mathcal{F}}.
	\end{equation}
	
	When the observation map is chosen in log-density or score space, the remaining differential term approximates the log-density drift:
	\begin{equation}
		\left[ \nabla_\theta M(D,\mathbf{p}) \right]^T \cdot (\theta_D-\theta_{D'}) \approx \ln \frac{q_D(\cdot|\mathbf{p})}{q_{D'}(\cdot|\mathbf{p})}.
	\end{equation}
	The pointwise magnitude of this log-density drift is controlled by the essential-supremum privacy budget $\epsilon(z_c,\mathbf{p})$ in Eq.~\ref{def: budget}. In the selected observation space, we summarize the same drift by the norm-based surrogate $\widehat{\epsilon}(z_c,\mathbf{p})$ in Eq.~\ref{def: budget_widehat}.
	Substituting this surrogate into Eq.~\ref{eq: decoupled_sensitivity} yields Eq.~\ref{eq: conditional_sensitivity_bound}.
\end{proof}

The bound shows that the observable sensitivity is governed by four quantities: the local step size $\eta$, the training-set scale $n$, the curvature of the empirical risk landscape $\lambda_{\max}(\mathbf{H})$, and the counterfactual exposure captured by $\widehat{\epsilon}(z_c,\mathbf{p})$. 
 A target sample that induces a larger counterfactual change in the conditional output distribution can therefore produce a larger gap, especially in regions where
the local risk landscape has high curvature. Conversely, if the target has little counterfactual effect on the model under condition $\mathbf{p}$, the induced dual-branch divergence is expected to remain small. In summary, conditional sensitivity provides an operational detection signal for measuring behavioral dependence on the target sample.

\subsection{Calibrated Detection Statistic} \label{sec: unified-5}

The theoretical bound motivates a post-hoc detection statistic based on the observable dual-branch behavioral gap. Given a target concept $z_c$ and a set of aligned conditions $\mathcal{P}_{c} \subseteq \mathcal{P}$ around $z_c$, we estimate the raw aligned sensitivity:
\begin{equation}
	S_c(z_c) \triangleq \mathbb{E}_{\mathbf{p} \sim \mathcal{P}_{c}} \left[ CS(z_c,\mathbf{p}) \right].
\end{equation}
In practice, $CS(z_c,\mathbf{p})$ can be computed in a modality-appropriate observation space, such as image embeddings, text logits, score-matching objectives, or other feature representations.

 A large value of $S_c(z_c)$ indicates that learning and unlearning perturbations around $z_c$ produce a measurable behavioral gap, which provides evidence that the target exerts non-negligible influence and explicit infringement on the model.

However, the dual-branch perturbation procedure intentionally perturbs model parameters. To separate concept-specific effects from global parameter shift, an auditor may evaluate an orthogonal conditioning set $\mathcal{P}_{\perp}$ whose prompts are semantically unrelated to $z_c$. The corresponding background sensitivity is:
\begin{equation}
	S_{\perp}(z_c) \triangleq \mathbb{E}_{\mathbf{p} \sim \mathcal{P}_{\perp}}
	\left[ CS(z_c,\mathbf{p}) \right].
\end{equation}
So the calibrated detection statistic (CDS) is then defined as:
\begin{equation}
	CDS(z_c) \triangleq S_c(z_c)-S_{\perp}(z_c).
\end{equation}
Equivalently,
\begin{equation}
	CDS(z_c) = \mathbb{E}_{\mathbf{p} \sim \mathcal{P}_{c}} 
	\left[ CS(z_c,\mathbf{p}) \right] - \mathbb{E}_{\mathbf{p} \sim \mathcal{P}_{\perp}} \left[ CS(z_c,\mathbf{p}) \right].
\end{equation}
This calibrated statistic is intended to retain target-specific sensitivity while reducing output changes caused by generic fine-tuning instability.
When a normalized statistic is preferred, one may use:
\begin{equation}
	\widetilde{CDS}(z_c) \triangleq \frac{S_c(z_c)} {S_{\perp}(z_c)+\tau},
\end{equation}
where $\tau>0$ is a small stabilizer. This form measures the aligned sensitivity relative to the background sensitivity induced by orthogonal conditions.

The calibration gives a direct statistical interpretation. $S_{\perp}(z_c)$ estimates the sensitivity that would be observed even if the target concept were irrelevant, whereas $S_c(z_c)$ measures sensitivity in the concept neighborhood. A reliable infringement signal should therefore satisfy two conditions: the aligned sensitivity is large, and the gap between aligned and orthogonal sensitivity is also large. This prevents the detector from confusing general model instability with target-specific memorization.

For practical detection, the final confidence score can be obtained by mapping the calibrated statistic through a monotone scoring function:
\begin{equation}
	f(z_c) = \sigma \left[ a \cdot CDS(z_c)+b \right],
\end{equation}
where $\sigma(\cdot)$ is the sigmoid function.
A high value of $f(z_c)$ indicates strong evidence that the target concept produces a conditional distributional shift beyond generic fine-tuning drift, leading to explicit copyright infringement.

\newpage
\section{Modality-Specific Instantiations}
\label{sec: modality-specific}
\subsection{Linear Regression Models}
We first instantiate the unified sensitivity framework in a ridge-regularized linear regression model (LRM)~\cite{montgomery2021introduction}. This case serves as a closed-form sanity check: every object in the previous section can be explicitly computed. To complement the derivation, we use the diabetes dataset from scikit-learn~\cite{pedregosa2018scikitlearnmachinelearningpython} for empirical verification.

\paragraph{Model Setup}
Let $D=\{(x_i,y_i)\}_{i=1}^n$, where $x_i \in \mathbb{R}^d$, $\|x_i\|_2 \leq 1$, and $|y_i| \leq 1$. Let $\mathbf{X}\in \mathbb{R}^{n\times d}$ be the design matrix with rows $x_i^\top$, and let $y=(y_1,\dots,y_n)^\top$. We consider the ridge objective:
\begin{equation}
	\mathcal{L}_{\mathrm{LRM}}(\beta;D) = \frac{1}{2n}\|\mathbf{X}\beta-y\|_2^2 + \frac{\rho}{2}\|\beta\|_2^2,
\end{equation}
where $\rho>0$ is the regularization coefficient. The corresponding mechanism outputs the unique minimizer:
\begin{equation}
	\hat{\beta}_D = \left(\frac{1}{n}\mathbf{X}^\top\mathbf{X}+\rho I\right)^{-1} \left(\frac{1}{n}\mathbf{X}^\top y\right).
\end{equation}
We then define the empirical Hessian as:
\begin{equation}
	\mathbf{H} \triangleq \nabla_\beta^2 \mathcal{L}_{\mathrm{LRM}}(\hat{\beta}_D;D)
	=\frac{1}{n}\mathbf{X}^\top\mathbf{X}+\rho I.
\end{equation}
Since $\lambda_{\min}(\mathbf{H})\geq \rho$, the inverse Hessian exists and is bounded by $\|\mathbf{H}^{-1}\|_2\leq \frac{1}{\rho}$.

\paragraph{Counterfactual Removal via Influence Functions}
Let $z_c=(x_c,y_c)$ be the target copyrighted sample and $D'=D\setminus\{z_c\}$. For the single-sample squared loss:
\begin{equation}
	\ell_c(\beta)=\frac{1}{2}(x_c^\top\beta-y_c)^2,
\end{equation}
the gradient at the trained model is:
\begin{equation}
	\nabla_\beta \ell_c(\hat{\beta}_D) = x_c r_c,
\end{equation}
where the residual is defined as $r_c \triangleq x_c^\top\hat{\beta}_D-y_c$. Since $\|x_c\|_2\leq 1$ and $|y_c|\leq 1$, and since ridge regularization gives $\frac{1}{\rho}$, the residual satisfies:
\begin{equation}
	|r_c| \leq |x_c^\top\hat{\beta}_D|+|y_c|
	\leq \frac{1}{\rho}+1.
\end{equation}
By employing the influence-function approximation from Eq.~\ref{eq: if_theta_approx}, the counterfactual parameter shift induced by removing $z_c$ satisfies:
\begin{equation}
	\hat{\beta}_D-\hat{\beta}_{D'} \approx -\frac{1}{n}\mathbf{H}^{-1}x_c r_c.
\end{equation}
Consequently,
\begin{equation}
	\|\hat{\beta}_D-\hat{\beta}_{D'}\|_2 \leq \frac{1}{n} \|\mathbf{H}^{-1}\|_2 \cdot \|x_c\|_2 \cdot |r_c|
	\leq
	\frac{1+\rho}{n\rho^2}.
\end{equation}

\begin{figure}
	\centering
	\includegraphics[width=\textwidth]{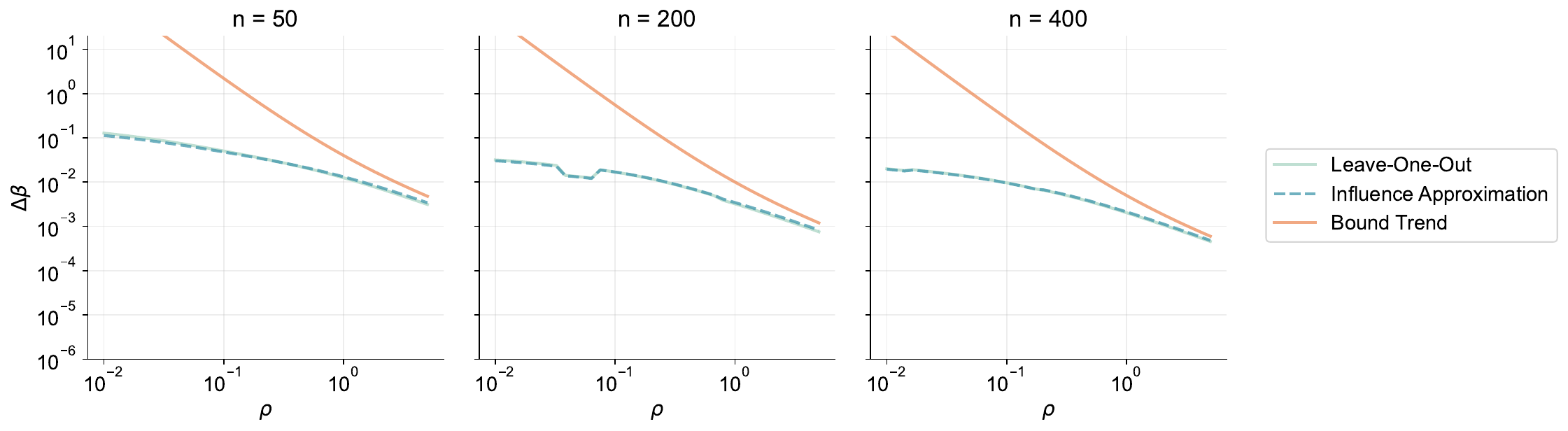}
	\caption{\textbf{Effect of Ridge Regularization on Sensitivity in LRM.} 
		The figure shows the change of parameter displacement $\Delta\beta=\|\hat{\beta}_D-\hat{\beta}_{D'}\|_2$ after removing a target sample under different regularization strengths $\rho$ and training-set sizes $n$.
		The exact leave-one-out solution closely matches the influence-function approximation, while the theoretical upper-bound trend decreases as $n$ and $\rho$ increase. This validates the closed-form prediction that ridge curvature stabilizes the model and suppresses the contribution of an individual sample.}
	\label{fig: vis_lrm}
\end{figure}

The resulting upper-bound trend, together with the exact leave-one-out displacement and the influence-function approximation, is shown in Figure~\ref{fig: vis_lrm}.

\paragraph{Linear Privacy-Budget Surrogate}
To connect this parameter displacement with the conditional privacy budget, suppose that the prediction at a test feature $x_{\mathrm{test}}$ follows a Gaussian observation model with variance $\sigma^2$:
\begin{equation}
	y_{\mathrm{test}}\mid D
	\sim
	\mathcal{N}(x_{\mathrm{test}}^\top\hat{\beta}_D,\sigma^2).
\end{equation}
For neighboring datasets $D$ and $D'$, define $\mu_D=x_{\mathrm{test}}^\top\hat{\beta}_D$ and $\mu_{D'}=x_{\mathrm{test}}^\top\hat{\beta}_{D'}$. Under the Gaussian
observation model, the pointwise log-likelihood ratio is:
\begin{equation}
	\log \frac{q_D(y|x_{\mathrm{test}})} {q_{D'}(y|x_{\mathrm{test}})} 
	= \frac{\mu_D-\mu_{D'}}{\sigma^2} \left( y-\frac{\mu_D+\mu_{D'}}{2} \right).
\end{equation}
For bounded outputs satisfying the inequality $\left|y-\frac{\mu_D+\mu_{D'}}{2}\right|\leq B_y$, this gives:
\begin{equation}
	\widehat{\epsilon}_{\mathrm{LRM}}(x_{\mathrm{test}})
	\lesssim \frac{B_y}{\sigma^2} \left| x_{\mathrm{test}}^\top (\hat{\beta}_D-\hat{\beta}_{D'}) \right|
	\approx \frac{B_y}{n\sigma^2} \left| x_{\mathrm{test}}^\top\mathbf{H}^{-1}x_c	\right| \cdot |r_c|.
\end{equation}

This expression is the linear-model counterpart of the counterfactual privacy budget in Eq.~\ref{def: budget}: the privacy exposure is governed by the residual $r_c$ and by the leverage of the target direction through the inverse Hessian.

\begin{figure}
	\centering
	\includegraphics[width=0.6\textwidth]{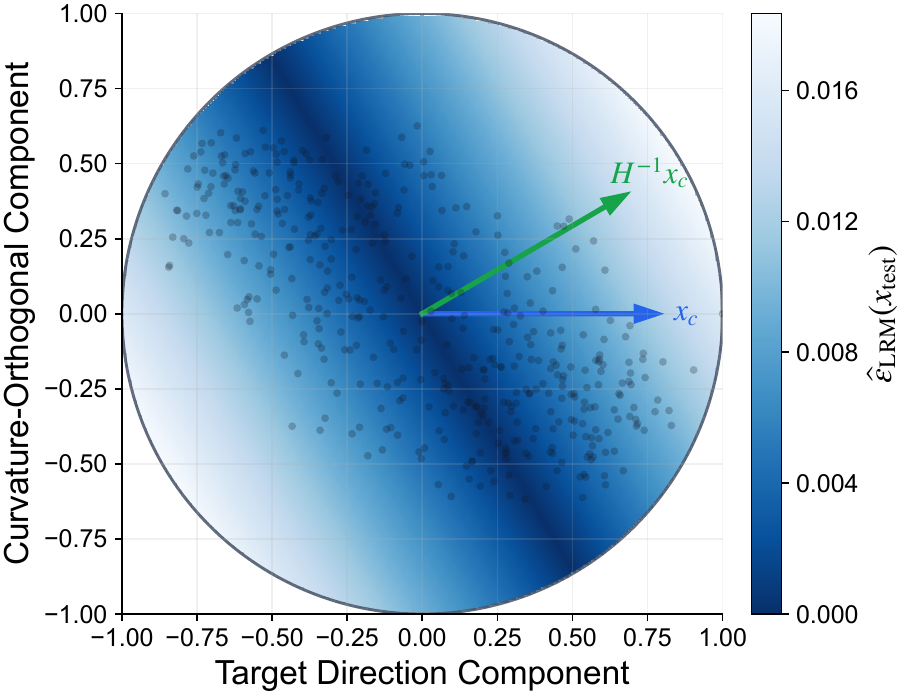}
	\caption{\textbf{Inverse-Hessian Governs Directional Privacy Exposure in LRM.} 
		The heatmap shows the privacy budget $\widehat{\epsilon}_{\mathrm{LRM}}(x_{\mathrm{test}})$ over a 2D slice spanned by the target direction $x_c$ and a curvature-orthogonal direction. Dark points denote real diabetes samples projected onto this slice.
		The exposure concentrates along $H^{-1}x_c$, rather than the raw target direction $x_c$, illustrating that the privacy budget is controlled by the residual magnitude and the targeted inverse-Hessian.}
	\label{fig: vis_lrm_budget}
\end{figure}

Figure~\ref{fig: vis_lrm_budget} visualizes the directional structure of the privacy-budget surrogate on the diabetes dataset.
The background heatmap plots $\widehat{\epsilon}_{\mathrm{LRM}}(x_{\mathrm{test}})$ over a two-dimensional slice spanned by the target direction and a curvature-orthogonal direction, while the dark points show the projected real samples.
The heatmap is symmetric because the surrogate depends on the magnitude $\left|x_{\mathrm{test}}^\top H^{-1}x_c\right|$, so projected diabetes samples exhibit an approximately two-sided empirical distribution.
Furthermore, the exposure concentrates along the inverse-Hessian direction $H^{-1}x_c$, rather than the raw target direction $x_c$, illustrating the role of local Hessian geometry in shaping privacy leakage.

\paragraph{Dual-Branch Perturbation}
The dual-branch fine-tuning construction admits an exact one-step closed form. Starting from the white-box model $\hat{\beta}_D$, we define:
\begin{equation}
	\beta_{D^+} = \hat{\beta}_D-\eta x_c r_c,
\end{equation}
\begin{equation}
	\beta_{D^-} = \hat{\beta}_D+\eta x_c r_c,
\end{equation}
which corresponds to Eq.~\ref{eq: learning_branch} and Eq.~\ref{eq: unlearning_branch}.
The induced behavioral gap at $x_{\mathrm{test}}$ is:
\begin{equation}
	CS_{\mathrm{LRM}}(z_c,x_{\mathrm{test}})
	\triangleq \left| x_{\mathrm{test}}^\top(\beta_{D^+}-\beta_{D^-}) \right|
	= 2\eta \cdot \left| x_{\mathrm{test}}^\top x_c \right| \cdot |r_c|.
\end{equation}

\begin{figure}
	\centering
	\includegraphics[width=\textwidth]{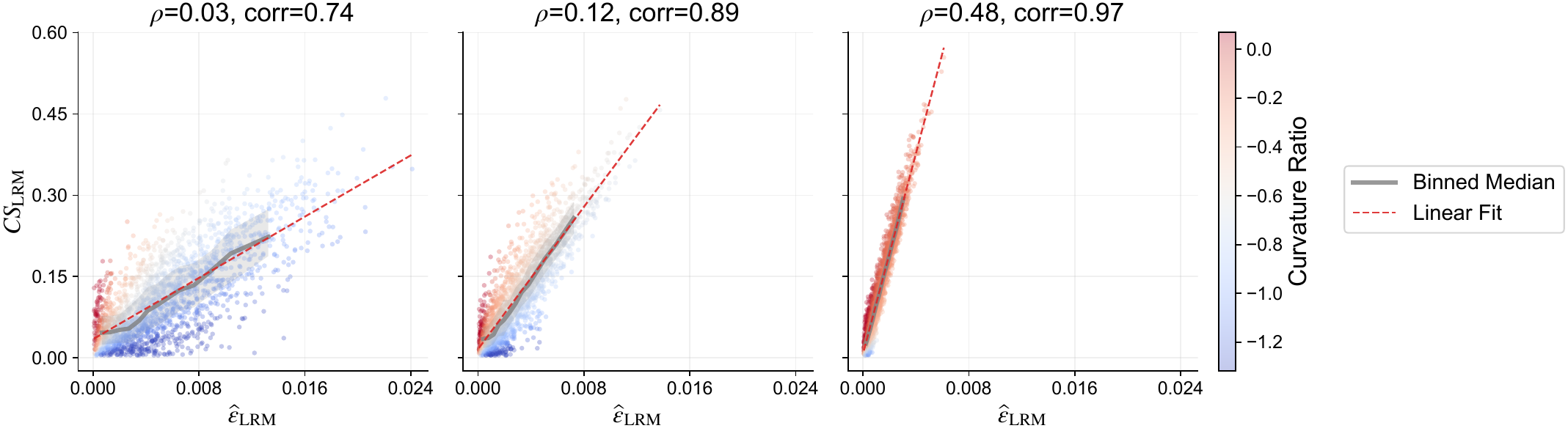}
	\caption{\textbf{Conditional Sensitivity and Privacy Budget Surrogate in LRM.} 
		It reports the relationship between dual-branch behavioral gap $CS_{\mathrm{LRM}}$ and the privacy-budget surrogate $\widehat{\epsilon}_{\mathrm{LRM}}$ under different regularization strength $\rho$.
		Each point corresponds to a target-test direction pair from the diabetes dataset.
		As $\rho$ increases, the correlation becomes stronger, indicating that stronger ridge regularization stabilizes Hessian geometry and makes $CS_{\mathrm{LRM}}$ a closer first-order surrogate for $\widehat{\epsilon}_{\mathrm{LRM}}$.}
	\label{fig: vis_lrm_csbudget}
\end{figure}

Figure~\ref{fig: vis_lrm_csbudget} empirically visualizes this first-order coupling on the diabetes dataset. Across different ridge strengths $\rho$, the dual-branch gap $CS_{\mathrm{LRM}}$ remains positively correlated with the privacy-budget surrogate $\widehat{\epsilon}_{\mathrm{LRM}}$. The coupling becomes tighter as $\rho$ increases, consistent with the fact that stronger regularization stabilizes the Hessian geometry and reduces the discrepancy between $x_c$ and $\mathbf{H}^{-1}x_c$.

\paragraph{First-Order Coupling}
Comparing the gap with the privacy budget gives:
\begin{equation}
	\frac{CS_{\mathrm{LRM}}(z_c,x_{\mathrm{test}})}
	{\widehat{\epsilon}_{\mathrm{LRM}}(x_{\mathrm{test}})}
	\propto
	2n\eta\sigma^2
	\frac{|x_{\mathrm{test}}^\top x_c|}
	{|x_{\mathrm{test}}^\top\mathbf{H}^{-1}x_c|}.
\end{equation}

When the detection direction lies in the target neighborhood, i.e., $x_{\mathrm{test}}\approx x_c$, this ratio is controlled by the Rayleigh quotient of Hessian in the target direction. Therefore, in linear models, the dual-branch behavioral gap and the counterfactual privacy budget are first-order equivalent up to the local curvature induced by the Hessian.

\paragraph{Limitation}
This closed-form derivation also clarifies a limitation of purely residual-based linear detection. If a linear model exactly interpolates the target sample and $r_c=0$, then a single gradient step on $\ell_c$ has zero first-order effect. In practice, this degeneracy can be avoided by evaluating a semantic neighborhood of the target, using nonzero perturbation residuals, or moving to richer observation spaces. The diffusion instantiation below follows this neighborhood-based strategy.

\subsection{Conditional Diffusion Models}
We now instantiate the same framework in conditional diffusion models (CDM), including denoising and latent diffusion families~\cite{ho2020denoising, rombach2022high}. It follows the dual-branch structure: the target concept is represented by a neighborhood of semantically aligned images and prompts, and the model is perturbed through learning and unlearning branches under the score-matching objective. This setting is closely related to concept editing and erasure methods that modify diffusion behavior around specified visual or textual concepts~\cite{gandikota2024unifiedconcepteditingdiffusion, nguyen2025sumasubspacemappingapproach}.
The efficiency of this dual-branch procedure has been validated in our prior work~\cite{Man_Wei_Chen_2026}.

\paragraph{Model Setup}
Let $z_0$ denote an image or latent representation associated with the target concept $z_c$, and let $\mathbf{p}\in\mathcal{P}$ be the condition such as a text prompt or image. A diffusion model~\cite{ho2020denoising} corrupts $z_0$ through a Gaussian forward process $z_t = \alpha_t z_0+\sigma_t \xi$, where $\xi\sim\mathcal{N}(0,I)$, $t$ denotes the timestep, and $\alpha_t,\sigma_t$ are the noise schedule coefficients. 

Let $\xi_\theta(z_t,t,\mathbf{p})$ be the denoising network that predicts the injected noise. The conditional score-matching loss for the target concept can be represented as:
\begin{equation}
	\mathcal{L}_{\mathrm{CDM}}(z_c;\theta)
	= \mathbb{E}_{z_0\in U(z_c),\xi,t,\mathbf{p}\sim\mathcal{P}_c}
	\left[ w_t
		\left\| \xi-\xi_\theta(\alpha_t z_0+\sigma_t\xi,t,\mathbf{p}) \right\|_2^2
	\right],
\end{equation}
where $U(z_c)$ is the semantic neighborhood of the target and $w_t$ denotes the weight. This loss is the diffusion analogue of the single-sample loss $\mathcal{L}(z_c;\theta)$ in the unified theory in Section~\ref{sec: unified-2}; related subject-driven customization methods also rely on small target neighborhoods to bind prompts to visual concepts~\cite{ruiz2023dreambooth, kumari2023multi}.

\paragraph{Dual-Branch Perturbation}
Given the pretrained parameter state $\theta_D$, the learning and unlearning branch perturbation in Eq.~\ref{eq: learning_branch} and Eq.~\ref{eq: unlearning_branch} is defined by:
\begin{equation}
	\theta_{D^+} = \theta_D-\eta\nabla_\theta \mathcal{L}_{\mathrm{CDM}}(z_c;\theta_D),
\end{equation}
\begin{equation}
	\theta_{D^-} = \theta_D+\eta\nabla_\theta \mathcal{L}_{\mathrm{CDM}}(z_c;\theta_D).
\end{equation}
Therefore, the parameter displacement of two branch parameters can be written as:
\begin{equation}
	\Delta\theta = \theta_{D^+}-\theta_{D^-} = -2\eta\nabla_\theta \mathcal{L}_{\mathrm{CDM}}(z_c;\theta_D).
\end{equation}

By the same influence-function argument as Eq.~\ref{eq: final_param_delta}, this displacement is coupled to the counterfactual drift $\theta_D-\theta_{D'}$ through the local curvature of the score-matching objective.

\paragraph{Observable Representation}

Since generated contents are high-dimensional and stochastic, the observable mechanism should be evaluated in a stable representation space. Let $\Phi$ denote a query function, such as an image encoder, text-image embedding model, or any domain-specific feature extractor. 
For a condition $\mathbf{p}$ and dataset $D$ corresponding to parameter $\theta_{D}$, we define:
\begin{equation}
	M(D,\mathbf{p}) \triangleq \Phi \left[ G_{D}(\mathbf{p}) \right].
\end{equation}
where $G$ denotes the generative sampler. The diffusion conditional sensitivity is then:
\begin{equation}	
	CS_{\mathrm{CDM}}(z_c,\mathbf{p})
	= \left\| M(D^+,\mathbf{p})- M(D^-,\mathbf{p}) \right\|_{\mathcal{F}}
	= \left\| \Phi \left[G_{D^+}(\mathbf{p}) \right] - \Phi \left[G_{D^-}(\mathbf{p}) \right] \right\|_{\mathcal{F}}.
\end{equation}
Its equivalent cosine form is:
\begin{equation}
	CS_{\mathrm{CDM, cos}}(z_c,\mathbf{p})
	= 1- \frac{ \Phi \left[G_{D^+}(\mathbf{p}) \right]^\top \cdot \Phi \left[G_{D^-}(\mathbf{p}) \right]
	}{ \left\|\Phi \left[G_{D^+}(\mathbf{p}) \right] \right\|_2 \cdot \left\|\Phi \left[G_{D^-}(\mathbf{p}) \right] \right\|_2 },
\end{equation}
which is mostly used in the CLIP image embeddings~\cite{radford2021learningtransferablevisualmodels}. It matches the practical dual-branch perturbation intuition: infringed concepts produce a larger branch divergence under aligned conditions, while non-infringed or generic concepts produce only minor changes.

Figure~\ref{fig: vis_cdm_sens} in our prior work~\cite{Man_Wei_Chen_2026} empirically illustrates this pattern across face, architecture, and artistic-painting categories in Stable Diffusion Model~\cite{rombach2022high}, indicating larger representation-space divergence between the learning and unlearning branches under aligned prompts.

\begin{figure}
	\centering
	\includegraphics[width=\textwidth]{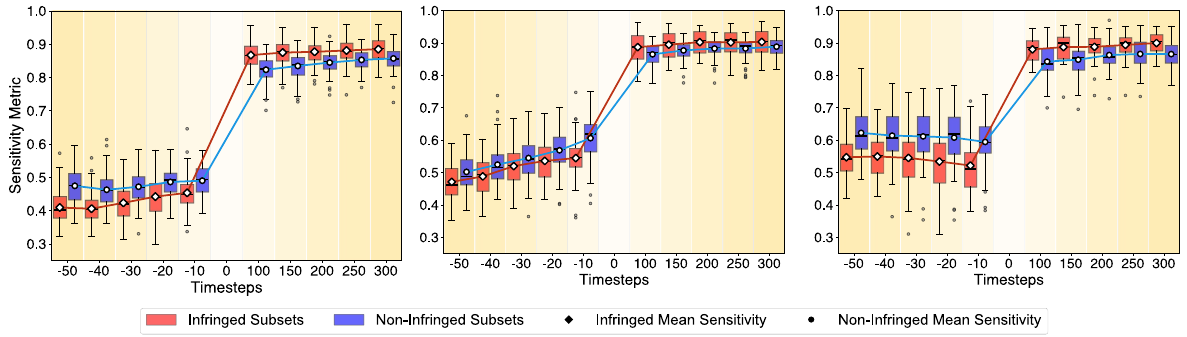}
	\caption{\textbf{Conditional Sensitivity (CLIP Image Encoder) across Categories in the Copyright Infringement Detection of Stable Diffusion.} Three categories are separately (left) human face, (middle) architecture, and (right) arts painting. Infringed samples are more sensitive to the model outputs behavior.}
	\label{fig: vis_cdm_sens}
\end{figure}

\paragraph{Source of the Score-Matching Signal}
The score-matching formulation also identifies the source of the signal. Let $ r_t(z_0,\mathbf{p};\theta) = \xi_\theta(\alpha_t z_0+\sigma_t\xi,t,\mathbf{p})-\xi
$ be the denoising residual. Then the target gradient can be written schematically as:
\begin{equation}
	\nabla_\theta \mathcal{L}_{\mathrm{CDM}}(z_c;\theta)
	= \mathbb{E}_{z_0,\xi,t,\mathbf{p}}
	\left[ 2w_t \cdot \left(\nabla_\theta \xi_\theta(\alpha_t z_0+\sigma_t\xi,t,\mathbf{p})\right)^\top
		\cdot r_t(z_0,\mathbf{p};\theta) \right].
\end{equation}

The gap measures how strongly the model's denoising field around the target neighborhood changes under opposite local updates. 
If the target concept induces strong model-specific dependence, the denoising field around $U(z_c)$ may become locally specialized, so opposite updates can induce a larger representation gap.
If the concept is absent or only weakly represented through generalization, the same perturbation is absorbed by the model's broader denoising prior, resulting in a smaller conditional sensitivity.

\paragraph{Calibrated Detection Statistic}
As diffusion fine-tuning can change global generation behavior even for conditions unrelated to the target, the calibrated score from the previous section is essential in this modality: aligned conditions from $\mathcal{P}_c$ estimate target-specific sensitivity, while orthogonal conditions from $\mathcal{P}_{\perp}$ estimate background fine-tuning drift. The difference between these two quantities yields a more stable post-hoc copyright detection statistic.

\begin{figure}
	\centering
	\includegraphics[width=\textwidth]{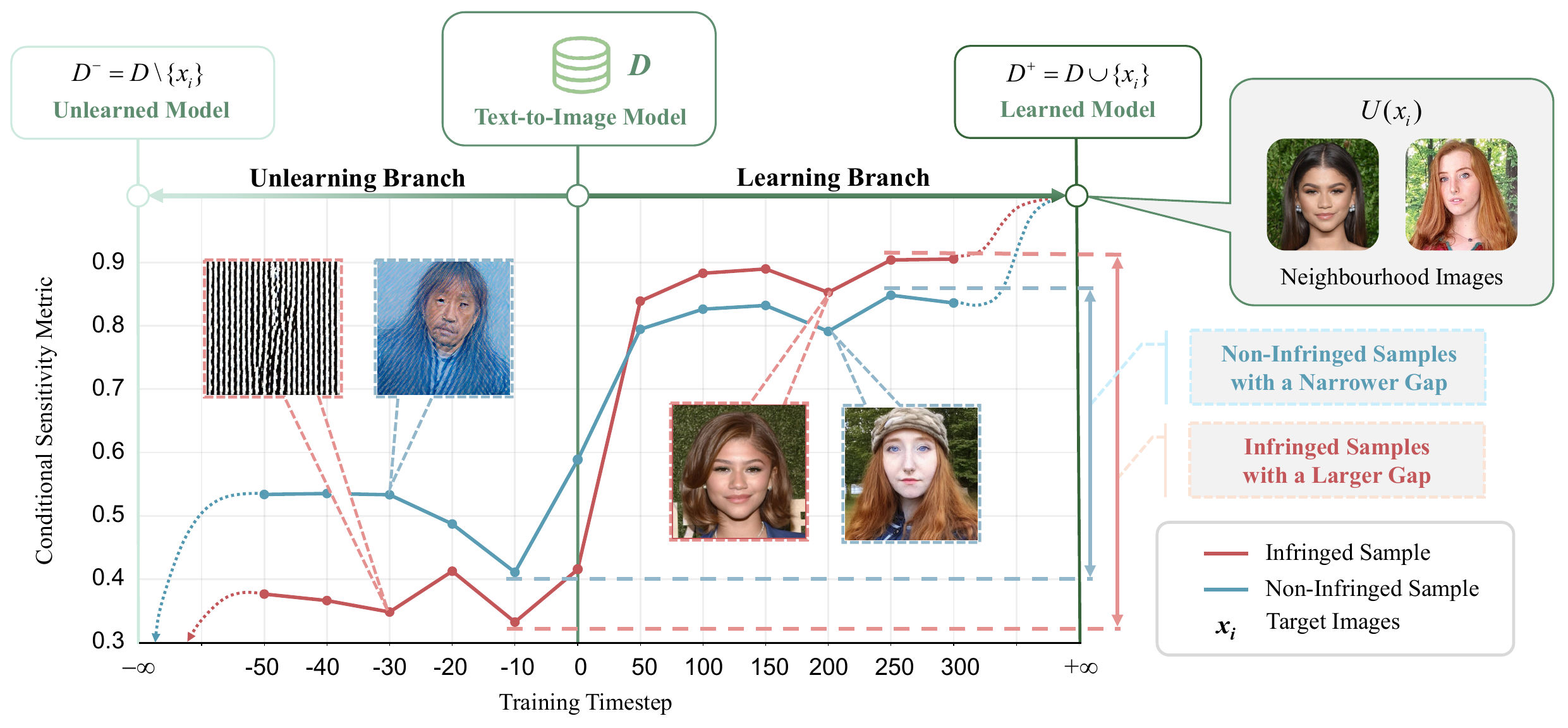}
	\caption{\textbf{Dual-Branch Perturbation in Conditional Diffusion Models.} Given a target neighborhood $U(z_c)$ and aligned conditions $\mathcal{P}_c$, the pretrained sampler $G$ is locally updated into two branches: the learning branch $G_{D^+}$ and the unlearning branch $G_{D^-}$. Conditional sensitivity measures the representation-space gap (e.g., CLIP image encoder) between two branch outputs.}
	\label{fig: vis_cdm}
\end{figure}

\paragraph{Visualization}
Figure~\ref{fig: vis_cdm} in our prior work~\cite{Man_Wei_Chen_2026} illustrates how the conditional diffusion instantiation operationalizes the dual-branch perturbation. Starting from a target concept neighborhood $U(z_c)$ and aligned conditions $\mathcal{P}_c$, the pretrained sampler is locally updated in two opposite directions: the learning branch decreases the target score-matching loss, while the unlearning branch increases it. 
The resulting samplers $G_{D^+}$ and $G_{D^-}$ are then evaluated under matched conditions, and their outputs are projected into a stable representation space through $\Phi$. A large representation gap indicates that the target neighborhood induces a strong condition-specific behavioral response, whereas the orthogonal calibration in Section~\ref{sec: unified-5} diminishes global drift unrelated to the target.

\subsection{Language Models}
\paragraph{Model Setup}
For autoregressive language models (LM), the mechanism $M$ maps a condition, which mostly denotes a prompt $\mathbf{p}=(w_1,\dots,w_m)$, to a conditional distribution over token sequences. Modern LMs are commonly built on transformer architectures~\cite{NIPS2017_3f5ee243}, with recent open foundation model families such as Qwen providing representative evaluation backbones~\cite{qwen2, qwen2.5}. Given parameters $\theta$, the probability of a continuation $y=(y_1,\dots,y_T)$ is factorized as:
\begin{equation}
	q_\theta(y|\mathbf{p}) = \prod_{t=1}^{T} q_\theta(y_t|\mathbf{p},y_{<t}).
\end{equation}
The target copyrighted data point $z_c$ may be a passage, code fragment, or stylistic expression; the modality-specific loss is the negative log-likelihood:
\begin{equation}
	\mathcal{L}_{\mathrm{LM}}(z_c;\theta) = -\sum_{t=1}^{T} \log q_\theta(y_t^c|\mathbf{p}_c,y^c_{<t}),
\end{equation}
where $(\mathbf{p}_c,y^c)$ denotes a prompt-continuation pair associated with the target. 

\paragraph{Dual-Branch Perturbation}
For language models, the dual-branch perturbation is instantiated by applying two opposite local updates with respect to the target negative log-likelihood. 
The learning branch performs a gradient descent step on the target sequence, and the unlearning branch performs the corresponding ascent step:
\begin{equation}
	\theta_{D^+} = \theta_D-\eta\nabla_\theta \mathcal{L}_{\mathrm{LM}}(z_c;\theta_D),
\end{equation}
\begin{equation}
	\theta_{D^-} = \theta_D+\eta\nabla_\theta \mathcal{L}_{\mathrm{LM}}(z_c;\theta_D).
\end{equation}
where $\eta>0$ is a perturbation step size. 

The purpose of this construction is to probe the local sensitivity of the model around $z_c$. 
If the target passage is strongly encoded in the model, then two updates in opposite directions should produce a measurable change in the conditional next-token distributions along the target prefix. 
If the passage is only weakly represented, or if the prompt activates general linguistic knowledge rather than target-specific memorization, the two branches should remain closer in decoding behavior.

\begin{figure}
	\centering
	\includegraphics[width=\textwidth]{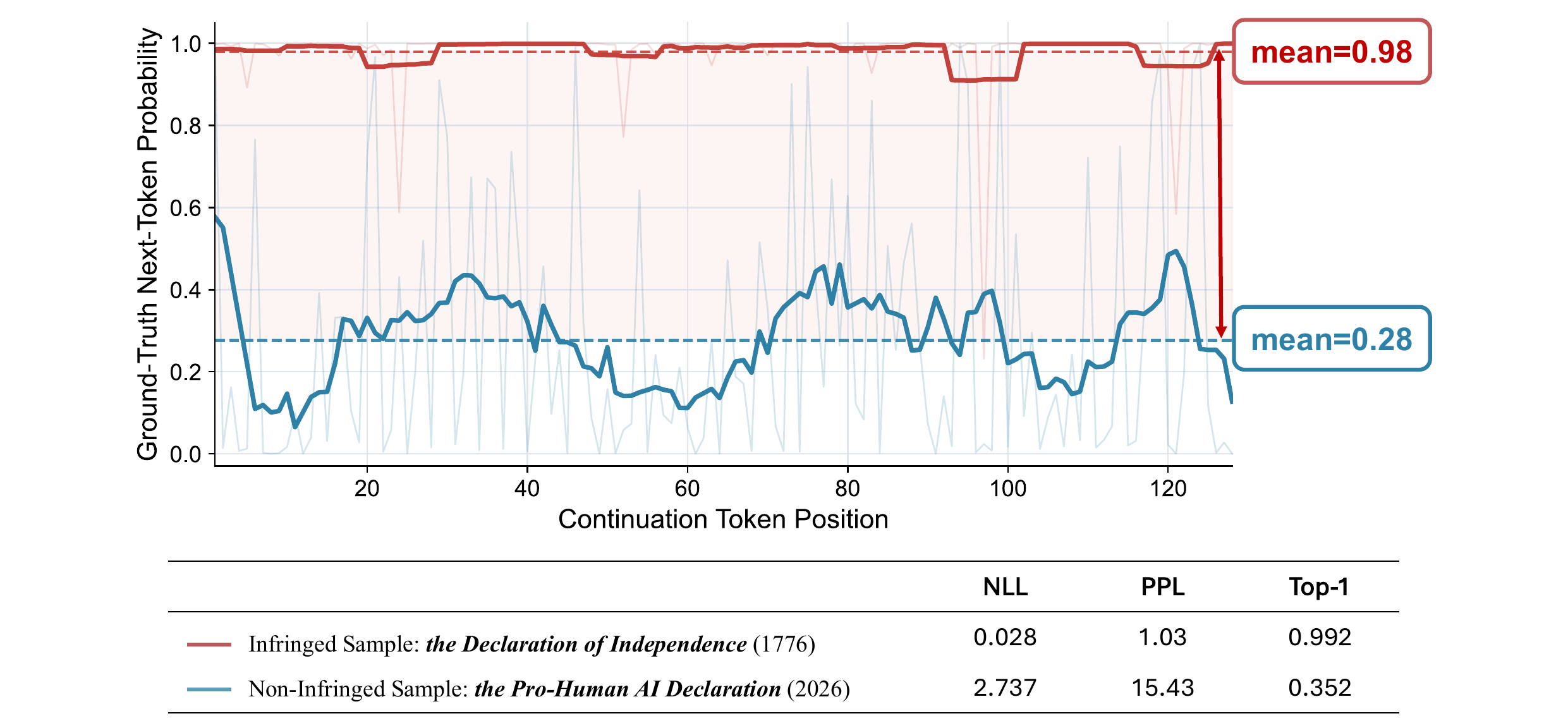}
	\caption{\textbf{Prefix-Conditioned Token Concentration in LM.} Given a fixed prefix prompt, Qwen2.5-7B assigns substantially higher ground-truth next-token probability to the infringed passage than to a non-infringed control passage. Dashed lines denote mean target-token probability over the continuation. The large confidence gap indicates that the infringed passage induces a near-deterministic continuation pattern, consistent with extractable memorization.}
	\label{fig: vis_lm}
\end{figure}

\paragraph{Observable Representation}
Unlike vision models, language models expose memorization through changes in token probabilities and entropy. For a probe prompt $\mathbf{p}$ aligned with the target, we define the next-token entropy at step $t$ as:
\begin{equation}
	H_\theta(t|\mathbf{p}) = -\sum_{v\in\mathcal{V}} q_\theta(v|\mathbf{p},y_{<t}) \log q_\theta(v|\mathbf{p},y_{<t}),
\end{equation}
where $\mathcal{V}$ is the vocabulary. A memorized continuation typically corresponds to an unusually concentrated token distribution along the target sequence. Therefore, the language-model conditional sensitivity can be measured by the divergence between token distributions in the two branches:
\begin{equation}
	CS_{\mathrm{LM}}(z_c,\mathbf{p}) =
	\frac{1}{T} \sum_{t=1}^{T} D_{\mathrm{KL}}
	\left[ q_{\theta_{D^+}}(\cdot|\mathbf{p},y_{<t}) \;\middle\|\; q_{\theta_{D^-}}(\cdot|\mathbf{p},y_{<t}) \right].
\end{equation}
An entropy-based alternative is:
\begin{equation}
	CS_{\mathrm{Ent}}(z_c,\mathbf{p}) =
	\frac{1}{T} \sum_{t=1}^{T} \left| H_{\theta_{D^+}}(t|\mathbf{p}) - H_{\theta_{D^-}}(t|\mathbf{p}) \right|.
\end{equation}

The KL form measures direct distributional disagreement, while the entropy form measures how much the certainty of the model changes under learning versus unlearning. If a target passage is memorized, the unlearning branch should disrupt the low-entropy continuation pattern more strongly than it would for a merely generic or unseen text. Hence a large token-distribution or entropy gap provides evidence of target-specific memorization or copyright infringement.

\paragraph{Calibrated Detection Statistic}
The same calibration principle applies to language models. Prompts in $\mathcal{P}_c$ should be semantically aligned with the target passage, while orthogonal prompts in $\mathcal{P}_{\perp}$ should be unrelated in topic, style, or named entities. The calibrated language score subtracts global decoding instability from target-specific token sensitivity.

\paragraph{Visualization} 
Figure~\ref{fig: vis_lm} visualizes the language-model instantiation through prefix-conditioned token concentration. Given an aligned prefix prompt from a target passage, the model is evaluated on the ground-truth continuation, and the probability assigned to each correct next token is recorded along the continuation positions. 

An infringed or memorized passage induces a highly concentrated decoding trajectory: the model assigns near-deterministic probability to most ground-truth next tokens, producing a low-entropy and low-perplexity continuation profile. By contrast, a non-infringed control passage produces a much less concentrated trajectory, with frequent probability drops and a lower average target-token probability. The vertical gap between the two mean probability levels is the language-model analogue of the representation-space gap used in diffusion models: it reflects target-specific conditional dependence under aligned prompts, while orthogonal prompts can be used to subtract background linguistic predictability.

\subsection{Multimodal Models}
For vision-language and multimodal generative models, a protected target may be encoded not only within a single modality, but also in the association between modalities. 
For example, a textual condition may activate a protected visual identity, a building name may retrieve a specific architectural appearance, or an artist identifier may induce a distinctive visual style. 
In such cases, output similarity alone may be insufficient for detecting infringement, because the relevant signal can appear as an internal binding between a condition and a protected visual, textual, or semantic concept. We therefore instantiate conditional sensitivity in a cross-modal representation space.

\paragraph{Model Setup}
Let $M$ denote a multimodal model that takes a condition $\mathbf{p}$ and, when applicable, an input or generated representation $z$. 
The model may produce outputs in one or more modalities, while internally maintaining joint representations that connect textual, visual, latent, or memory tokens. 
We denote such a cross-modal representation by:
\begin{equation}
	\Psi_D(\mathbf{p},z) \in \mathcal{H}_{\mathrm{cross}},
\end{equation}
where $\mathcal{H}_{\mathrm{cross}}$ is a shared representation space for measuring modality-level associations. 
Depending on the architecture, $\Psi_D$ corresponds to a joint image-text embedding, a multimodal hidden state, or a latent conditioning vector.

To construct the two perturbation branches, the target concept must be paired with a modality-specific cross-modal objective. We write this objective abstractly as:
\begin{equation}
	\mathcal{L}_{\mathrm{cross}}(z_c;\theta)
	=
	\mathbb{E}_{(\mathbf{p},z)\sim U(z_c)}
	\left[
		\ell_{\mathrm{cross}}(\mathbf{p},z;\theta)
	\right],
\end{equation}
where $U(z_c)$ denotes the target neighborhood containing aligned text, image, audio, latent, or other modality pairs, and $\ell_{\mathrm{cross}}$ may be instantiated as a contrastive image-text loss, captioning negative log-likelihood, image-text matching loss, multimodal reconstruction loss, or another architecture-specific binding objective. This loss plays the same role as $\mathcal{L}_{\mathrm{CDM}}$ in diffusion models and $\mathcal{L}_{\mathrm{LM}}$ in language models.

\paragraph{Dual-Branch Perturbation}
As in the previous modalities, we construct two local branches around the pretrained parameter point $\theta_D$ by taking opposite steps on the cross-modal target loss:
\begin{equation}
	\theta_{D^+} = \theta_D-\eta\nabla_\theta \mathcal{L}_{\mathrm{cross}}(z_c;\theta_D),
\end{equation}
\begin{equation}
	\theta_{D^-} = \theta_D+\eta\nabla_\theta \mathcal{L}_{\mathrm{cross}}(z_c;\theta_D),
\end{equation}
where the learning branch strengthens the target association and the unlearning branch weakens it. Their displacement is therefore:
\begin{equation}
	\Delta\theta = \theta_{D^+}-\theta_{D^-}
	= -2\eta\nabla_\theta \mathcal{L}_{\mathrm{cross}}(z_c;\theta_D).
\end{equation}
Given these two perturbed branches, we define cross-modal conditional sensitivity as:
\begin{equation}
	CS_{\mathrm{cross}}(z_c,\mathbf{p})
	= \left\| \Psi_{D^+}(\mathbf{p},z_c) - \Psi_{D^-}(\mathbf{p},z_c) \right\|_{\mathcal{H}_{\mathrm{cross}}}.
\end{equation}
This statistic measures how strongly the learned correspondence between the condition and the protected concept changes under opposite perturbations. 

If the model has memorized a specific cross-modal association, such as a character name with a distinctive face, a building name with a protected architectural image, or an artist identifier with a protected style, then the two branches should produce a larger representation gap around aligned prompts. 
If the association is weak, generic, or produced mainly by broad semantic generalization, the corresponding gap should remain smaller.

\paragraph{Attention as a Special Case}

For architectures with explicit cross-attention~\cite{NIPS2017_3f5ee243}, the general representation map $\Psi_D$ can be instantiated by attention tensors. Let $a_D^{(\ell,h)}(\mathbf{p},z)$ denote the cross-attention matrix at layer $\ell$ and head $h$, where tokens from one modality attend to tokens from another modality. 
Collecting these matrices over selected layers and heads gives:
\begin{equation}
	\mathcal{A}_D(\mathbf{p},z) = \{a_D^{(\ell,h)}(\mathbf{p},z)\}_{\ell,h}.
\end{equation}
The attention-based sensitivity score is then written as:
\begin{equation}
	CS_{\mathrm{attn}}(z_c,\mathbf{p}) = \left\| \mathcal{A}_{D^+}(\mathbf{p},z_c) - \mathcal{A}_{D^-}(\mathbf{p},z_c) \right\|_{\mathcal{F}}.
\end{equation}

Thus, attention sensitivity is an architectural realization of cross-modal conditional sensitivity, rather than the only possible definition. 
This distinction is important for models that use joint-token transformers, contrastive encoders, latent conditioning modules, or decoder-only multimodal architectures, where the relevant binding may not be represented by an explicit cross-attention matrix.

\paragraph{Joint Output--Representation Score}
Cross-modal sensitivity can also be combined with output-space sensitivity. 
Let $\Phi_{\mathrm{out}}$ be an encoder for the generated or predicted output, and let $\Psi_D$ be the selected internal cross-modal representation. 
A joint score can be written as
\begin{align}
	CS_{\mathrm{joint}}(z_c,\mathbf{p})
	= &\alpha \left\| \Phi_{\mathrm{out}} \left[ M(D^+,\mathbf{p}) \right] - \Phi_{\mathrm{out}} \left[ M(D^-,\mathbf{p}) \right] \right\|_{\mathcal{F}} \\
	&+ (1-\alpha) \left\| \Psi_{D^+}(\mathbf{p},z_c) - \Psi_{D^-}(\mathbf{p},z_c) \right\|_{\mathcal{H}_{\mathrm{cross}}},
\end{align}
where $\alpha\in[0,1]$ controls the balance between observable output divergence and internal cross-modal representation divergence. 
The output term captures visible behavioral changes, while the representation term captures changes in the internal association between the condition and the protected concept. 

This formulation is useful when generated samples are noisy, stochastic, or visually ambiguous, but the model's internal cross-modal representation still reveals stable concept-specific dependence.

\paragraph{Calibrated Detection Statistic}
The same orthogonal calibration principle applies in the cross-modal setting.  Aligned prompts in $\mathcal{P}_c$ should activate the suspected protected association, while orthogonal prompts in $\mathcal{P}_{\perp}$ should preserve the general modality structure without referring to the protected concept. The calibrated score subtracts background multimodal instability from target-specific cross-modal sensitivity, reducing the risk that generic prompt sensitivity or sampling variance is mistaken for memorization.

\newpage
\section{Conclusion}
\label{sec: conclusion}

This paper presents \textbf{Dual-Branch Conditional Sensitivity (DCS)}, a unified framework for post-hoc copyright infringement detection in generative and multimodal models. Instead of treating infringement as a static similarity judgment between an output and a copyrighted reference, we formulate it as a counterfactual conditional distribution shift. Under this view, evidence of infringement arises when the inclusion or exclusion of a protected target produces a measurable change in the model's behavior under aligned conditions.

The proposed theory connects copyright memorization with conditional differential privacy. Conditional publicity characterizes the case in which a target concept produces a substantial privacy violation, while non-infringement corresponds to approximate invariance with respect to the target. This formulation clarifies why distinctive long-tail works can be more auditable than generic concepts, and why output-only tests must be interpreted under explicit assumptions about target salience, observable model behavior, and access to controlled conditioning spaces.

To make the counterfactual criterion operational, we introduce \textbf{conditional sensitivity} based on a dual-branch perturbation procedure. The learning branch and unlearning branch approximate opposite local movements around the target concept, and their observable divergence serves as a surrogate for the unavailable retraining-based counterfactual. The first-order analysis shows that this divergence is controlled by the counterfactual privacy-budget surrogate, the training-set scale, the perturbation step size, and the local Hessian curvature. Furthermore, the calibrated detection statistic subtracts sensitivity under orthogonal conditions, reducing false evidence caused by global parameter shift or sampling instability.

The modality-specific instantiations demonstrate that the same principle can be applied beyond a single model family. Linear regression provides a closed-form sanity check in which residuals and inverse-Hessian geometry govern exposure. Conditional diffusion models express the signal through image or latent representation divergence. Language models expose it through token-distribution and entropy shifts. Multimodal models extend the criterion to cross-modal representation and attention sensitivity, capturing protected associations that may not be visible from raw outputs alone.

The framework also has clear limitations. It does not claim to decide legal infringement by itself, nor does it eliminate the need for domain-specific evidence about protectable expression, authorization, and substantial similarity. 
Its statistical reliability depends on the distinctiveness of the target, the quality of aligned and orthogonal conditions, the stability of the perturbation procedure, and the observability of the chosen representation.

\newpage
\bibliographystyle{unsrt}
\bibliography{references.bib}

\end{document}